\documentclass[12pt]{article}
\pdfoutput=1 

\usepackage[english]{babel}
\usepackage[letterpaper,top=2cm,bottom=2cm,left=3cm,right=3cm,marginparwidth=1.75cm]{geometry}

\usepackage{amsmath}
\usepackage{graphicx}
\usepackage[colorlinks=true, allcolors=blue]{hyperref}
\usepackage[utf8]{inputenc}
\usepackage{url}
\usepackage{hyperref}
\usepackage{amsmath}
\usepackage{amssymb}
\usepackage{accents}
\usepackage[english]{babel}
\usepackage{amsthm}
\usepackage{geometry,mathtools}
\usepackage{esvect}
\usepackage{dsfont}
\usepackage{caption}
\usepackage{subcaption}
\usepackage{tikz}
\usetikzlibrary{automata,arrows,positioning,calc}
\usepackage{blkarray}
\usepackage{authblk}
\usepackage{mathtools}
\mathtoolsset{showonlyrefs}
\usepackage{enumerate}
\usepackage{enumitem}   

\usepackage{algorithm}
\usepackage{algorithmic}
\usepackage{caption}
\usepackage{subcaption} 

\usepackage{comment}
\usepackage{amsthm}
\usepackage{tikz}
\usepackage{amsfonts}
\usepackage{dsfont}
\usepackage{blkarray}
\usepackage{hyperref}
\usepackage{amssymb}
\usepackage{accents}

\newcommand{\R}{\mathbb{R}}

\newcommand{\cV}{{\mathcal{V}}}

\newtheorem{thm}{Theorem}
\newtheorem{prop}{Proposition}

\newtheorem{lemma}{Lemma}

\theoremstyle{definition}

\title{Control Disturbance Rejection in Neural ODEs}

\author{Erkan Bayram\textsuperscript{*}, Mohamed-Ali Belabbas\textsuperscript{*}, Tamer  Ba{\c{s}}ar\thanks{E. Bayram, M.-A. Belabbas, and T. Başar are with the Electrical and Computer Engineering Department and with the Coordinated Science Laboratory, University of Illinois Urbana-Champaign, Urbana, IL, 61801. E-mail: \emph{(ebayram2,belabbas,basar1)@illinois.edu}.
Research was supported in part by the ARO Grant W911NF-24-1-0085, NSF-CCF 2106358, ARO W911NF-24-1-0105 and AFOSR FA9550-20-1-0333}}
\date{}

\begin{document}
\maketitle
\begin{abstract}
In this paper, we propose an iterative training algorithm for Neural ODEs that  provides models resilient to control (parameter) disturbances.  The method builds on our earlier work Tuning without Forgetting---and similarly introduces training points sequentially, and updates the parameters on new data within the space of parameters that do not decrease performance on the previously learned training points---with the key difference that, inspired by the concept of flat minima, we solve a minimax problem for a non-convex non-concave functional over an infinite-dimensional control space. We develop a projected gradient descent algorithm on the space of parameters that admits the structure of an infinite-dimensional Banach subspace. We show through simulations that this formulation enables the model to effectively learn new data points and gain robustness against control disturbance.


\end{abstract}

\section{Introduction}

The limit of residual neural networks as the number of layers goes to infinity naturally leads to neural ordinary differential equations (nODEs)~\cite{dupont2019augmented}. This point of view allows one to bring the theory of controlled dynamical systems to bear on learning and inference problems~\cite{weinan2017proposal}. 

The model we consider is a controlled differential equation (i.e. nODE) $ \dot x = f(x,u) $,
where the learning task is to find a control $u^*$ so that the flow $\varphi_t(u^*,\cdot)$ generated by the nODE satisfies 
\begin{align}\label{eqn:endpoint}
   R(\varphi_T(u^*,x^i))=y^i, \mbox{ for all } (x^i,y^i)\in (\mathcal{X},\mathcal{Y})
\end{align}
for some finite $T\geq0$, a given fixed readout map $R$ and  where $(\mathcal{X},\mathcal{Y})$ is a finite set of input-output (feature-labels) pairs with cardinality $q$. We call $R(\varphi_T(u^*,\cdot))$ the {\em end-point mapping} (i.e., the forward equation of the neural ODE with parameter $u^*$ through the readout map \(R\)). A typical approach to solve the above learning problem~\cite{agrachev2020control} is to embed all points in $\mathcal{X}$ into a higher-dimensional space under a single controlled ODE (CODE), known as the $q$-folded method~\cite{agrachev2020control}.

In this paper, we introduce the {\em iterative training algorithm for robust nODE} as an alternative to empirical risk minimization and the $q$-folded method, that allows training of a robust neural ODE against the presence of control disturbance. Due to the sequential introduction of points, an iterative training algorithm should ensure that the model retains the end-point mappings for the previously seen data (i.e., not forget these samples) while it maintains the ability to learn new points. The restriction on the training regime stemming from not forgetting previously learned points has been discussed from a different perspective in the learning literature; learning new points while preventing catastrophic forgetting, commonly referred to as continual learning, has been addressed using various strategies. Some methods aim to mitigate catastrophic forgetting focus on finding {\em flat} local minima in the cost functional during training, allowing for the incremental learning of new points (tasks) within these stable regions~\cite{mirzadeh2020understanding}. A flat local minimum is obtained at parameters for which the cost function remains close to the local minimum within a relatively large neighborhood of the said minimum~\cite{shi2021overcoming}. Some techniques have been proposed to facilitate the search for flat minima such as Sharpness-Aware Minimization (SAM)~\cite{foret2020sharpness}.

The search for flat minima naturally leads to a  robust learning problem for nODEs, expressed as the following minimax optimal control problem:
\begin{equation}
    \min_{u \in \cV}\sum_{i=1}^q  \left( \max_{||\varepsilon||_\infty \leq \rho} ||R(\varphi_T(u + \varepsilon,x^i))-y^i||^2_2 \right),
\end{equation}
where $\cV=L^\infty([0,1],\mathbb{R}^p)$. While previous works~\cite{cipriani2024minimax,yan2019robustness} have addressed disturbances (noise) in the training samples within a finite-dimensional space, we tackle the more challenging problem of training a nODE that is robust against disturbances $\varepsilon$ in the control $u \in \cV$ (model parameters). The difficulty in solving this problem stems from lack of compactness of the infinite-dimensional optimization space. The difficulty is compounded with the non-convexity of the objective function and the infinite-dimensional nature of disturbances, which makes the problem harder to solve.

Our approach is to first show that there exists a compact, finite-dimensional subspace of the space of control functions that shares the same maximum value of the inner maximization problem. Then, we derive a closed-form solution for the inner maximization problem as a function of $u$ using calculus of variations. This allows us to convert the non-convex non-concave minimax problem over the control space into a more manageable form.

To solve the outer minimization problem, we present data points sequentially to our algorithm and adjust (or tune) the control function $u$ to $\tilde u$ (the model) to learn the latest introduced pair $\{(x^j,y^j)\}$ while preserving previously acquired knowledge for the points $\{(x^i,y^i)\}_{i=1}^{j-1}$. At every update of the control $u$, the set of admissible controls is restricted to those that match all the previously learned pairs $\{(x^i,y^i)\}_{i=1}^{j-1}$ (i.e. $R(\varphi_T(u,x^i))=y^i, \forall i < j$). In our earlier work~\cite{bayram2024control}, we have discussed a similar restriction for the set of admissible control functions. To handle this constraint, we have proposed a kernel projection method. {Our main contributions in this paper are the followings:}
\begin{enumerate}
    \item We propose a novel iterative training algorithm for neural ODEs, that provides models (controls) resilient to control disturbance. We call our method {\em iterative training algorithm for robust nODE}.
    \item We solve a minimax optimization problem for a non-convex non-concave functional over an infinite-dimensional control space. We show that this formulation allows the model to effectively learn the new data point and gain robustness against control disturbance. 
    \item Building on our earlier works~\cite{bayram2024control,bayram2025geometric}, we introduce training points sequentially to our algorithm. In this work, we employ a kernel projection method onto an infinite-dimensional Banach subspace of \(\mathcal{V}\). {We demonstrate how this subspace can be used to reject control disturbance in neural ODEs.}
    \item We evaluate our method on a classification task and compare its performance under varying magnitudes of disturbances to a standard nODE.
\end{enumerate}

\section{Preliminaries}
Consider a paired dataset $(\mathcal{X},\mathcal{Y}) = \{(x^i, y^i)\}_{i=1}^q$, where $x^i \in \mathbb{R}^n$ are distinct initial points and their corresponding $y^i \in \mathbb{R}^{n_o}$ are targets. Let $\mathcal{I} = \{1,\dots,q\}$ be index set for the entries $\mathcal{X}$. For $j>0$, define the sub-ensemble $\mathcal{X}^j = \{x^i \in \mathcal{X} \mid i = 1,\dots,j\}$ with corresponding labels $\mathcal{Y}^j$, and set $\mathcal{X}^0 = \mathcal{Y}^0 = \emptyset$. Let $\cV := L^{\infty}([0,1], \mathbb{R}^{p})$. We take the system:
\begin{equation}\label{eqn:control_system}
    \Dot{x}(t) = f( x(t) , u(t) )
\end{equation}
where $x(t)\in\mathbb{R}^{\bar{n}}$ is the state vector at time $t$ and $f(\cdot)$ is a smooth vector field on $\mathbb{R}^{\Bar{n}}$ and {$u(t) \in \cV$.

The flow of this system defines the map  $\varphi_t(u,x): \cV \times \mathbb{R}^{\bar{n}}  \to \mathbb{R}^{\bar{n}}$ which assigns an initial state $x$ and a control $u$ to the solution of \eqref{eqn:control_system} at $t$. We suppress the subscript $1$ in the notation at $t=1$ for simplicity. Suppose that~\eqref{eqn:control_system} has uniformly bounded $\frac{\partial f(x,u)}{\partial u}$ and $\frac{\partial f(x,u)}{\partial x}$ for $t \in [0,1]$ where $x=\varphi_t(u,x^i)$. We introduce the following map to embed an $n$-dimensional space into a $\bar n$-dimensional one, where $\bar n \geq n$:
\begin{align}\label{eqn:uplift}
    E: \mathbb{R}^{n} \to \R^{\bar n}: x \mapsto E(x):=(x,0,\ldots,0).
\end{align}
Let $R:\mathbb{R}^{\bar{n}} \to \mathbb{R}^{n_o}$ be a given function, called the {\em readout map} such that the Jacobian of $R$ is of full row rank, and $R(\cdot)$ is bounded, linear and $1$-Lipschitz. (any projection satisfies this). Both $E$ and $R$ are independent of the control $u$. Using $E$, the end-point mapping becomes an augmented nODE~\cite{dupont2019augmented}. For simplicity, we take $E$ as the identity ($\bar{n}=n$), though results hold for any $\bar{n} \ge n$. We call $R(\varphi(u,E(\cdot)))$ {\em the end-point mapping}, that is, we have the following end-point mapping for a given $u$ at $x^i$:
\begin{align}
    x^i \in \mathbb{R}^{n} \xrightarrow{E} \Bar{x}^i \in \mathbb{R}^{\Bar{n}}  \xrightarrow{\varphi_T(u,\cdot)} \Bar{y}^i \in \mathbb{R}^{\bar{n}} \xrightarrow{R} y^i  \in \mathbb{R}^{n_o} 
\end{align}
The learning problem then turns into finding an {open-loop} control function $u$ that performs \textit{simultaneously} motion planning for initial points $\bar{x}^i$ to a point in the set $\{ \bar{y}^i \in \mathbb{R}^{\bar{n}} : R(\bar{y}^i) = y^i\}$ for all $i \in \mathcal{I}$ . 

\paragraph{Cost Functional}
 Let $$\mathcal{V}_\rho=\{ v \in L^\infty([0,1], \mathbb{R}^p) : ||v||_\infty \leq \rho \}$$ for some bounded $\rho \geq 0$. We call the set $\mathcal{V}_\rho$ disturbance set. We are interested in the minimization of the functional $\cal J:\cal V \to \mathbb{R}$, defined as
\begin{align}\label{eqn:cost_cont}
    \mathcal{J}( u , \mathcal{X} ) := \sum_{i=1}^{q} \left(  \max_{\varepsilon \in \mathcal{V}_\rho}  \mathcal{J}^{\mbox{Robust}}_i(u,\varepsilon)  \right) 
\end{align} 
where we define {\em robust cost functional}, denoted by $\mathcal{J}^{\mbox{Robust}}_i(u,\varepsilon)$ for a given pair $(x^i,y^i)$, as:
\begin{align}\label{eqn:robust_cost}
    \mathcal{J}^{\mbox{Robust}}_i(u,\varepsilon) := \| R(\varphi(u + \varepsilon, x^i)) - y^i \|^2  - \lambda_1 ||\varepsilon||_2^2 
\end{align}
where $\lambda_1>0$ is some regularization coefficient. We refer to the maximization of $\mathcal{J}^{\mbox{Robust}}_i(u,\cdot)$ for a given control $u$ over $\varepsilon \in \cV_\rho$ as the inner maximization.

\section{Main Result}
First, we show that there exists a unique $ \varepsilon^*_i $ that maximizes a first-order approximation of $ \mathcal{J}^{\mbox{Robust}}_i(u,\cdot) $ for a given pair $(x^i,y^i)$ and a control function $u \in \cV$. Next, we substitute this result into~\eqref{eqn:cost_cont}. Finally, we introduce our iterative algorithm to minimize an approximation of $ \mathcal{J}(u,\cal X) $ over $ u \in \cV $.

\paragraph{Maximize $\mathcal{J}^{\mbox{Robust}}_i(u,\cdot)$} 
 
We seek to maximize the functional~\eqref{eqn:robust_cost} over the disturbance set \( \mathcal{V}_\rho \). However, lack of compactness of the set \( \mathcal{V}_\rho \) prevents the existence of a solution that satisfies the optimality conditions derived from the first- and second-order variations of \( \mathcal{J}^{\mbox{Robust}}_i(u,\varepsilon) \).

Our first result aims to address this issue. To this end, we consider the trajectory  $\varphi_t(u,x^i)$. It is well known that the first-order variation in $\varphi_t(u, x^i)$, denoted by $\delta \varphi_t(u, 
 x^i)=\varphi_t(u+ \delta u, x^i)-\varphi_t(u,x^i)$, obeys a linear time-varying equation, which is simply the linearization of the control system~\eqref{eqn:control_system} around the trajectory $\varphi_t(u, x^i)$~\cite[Lemma 2]{bayram2024control}. For completeness, we recall the LTV system:
 \begin{align}\label{eqn:defn_ltv}
    \dot{z}(t) =  \frac{\partial f(x,u)}{\partial x} z(t) + \frac{\partial f(x,u)}{\partial u}  \delta u(t),
\end{align}
where $z(t)=\delta \varphi_t(u,x^i)$, $x(t)=\varphi_{t}(u,x^i)$. Denote the state transition matrix of the system in~\eqref{eqn:defn_ltv} by $\Phi_{(u, x^i)}(1,t)$ for initial point $x^i$ at control $u$. {We define an operator from the space of bounded functions over the time interval $[0,1]$, $v \in \cV$, to $\mathbb{R}^{n_o}$, mapping a control variation to the resulting variation in the end-point of the trajectory:
\begin{align}\label{eqn:operator_L}
    \mathcal{L}_{(u,x^i)}(v) = \left( \int_{0}^{t}  \Phi_{(u,x^i)}(t,\tau)  \frac{\partial f( x, u ) }{\partial u}  v(\tau)  d\tau \right)
\end{align}
where $x=\varphi_{\tau}(u,x^i)$. 
}

\begin{lemma}\label{lem:reach_set}
Consider the operator $\mathcal{L}_{(u,x^i)}(\cdot)$ for a given $u$ and an initial point $x^i$. For all $\rho>0$, there exists a finite-dimensional compact subspace of $\mathcal{V}_\rho$, denoted by $\mathcal{V}^*_{\rho}$, such that $\mathcal{L}_{(u,x^i)}({\mathcal{V}^*_\rho})=\mathcal{L}_{(u,x^i)}({\mathcal{V}_\rho})$.
\end{lemma}
See Section~\ref{sec:proof} for the proof of Lemma~\ref{lem:reach_set}. The statement says that, under given conditions, we have a compact subset of $\cV_\rho$, denoted by $\cV^*_\rho$, that has the same range through the operator $\mathcal{L}_{(u,x^i)}(\cdot)$. Then, from~\cite[Lemma 2]{bayram2024control} and the definition of the operator $\mathcal{L}_{(u,x^i)}(\cdot)$~\eqref{eqn:operator_L}, we have:
\begin{align}\label{eqn:disturbance_cost}
    R( \varphi(u+\varepsilon,x^i) ) =   R( \varphi(u,x^i) ) + R(\mathcal{L}_{(u,x^i)}(\varepsilon))
\end{align}
up to the first order. Let $r^u_i := R(\varphi(u,x^i)) - y^i$. When we plug~\eqref{eqn:disturbance_cost} into~\eqref{eqn:robust_cost}, we have: 
\begin{align}\label{eqn:robust_cost_L}
    \tilde{\mathcal{J}}^{\mbox{Robust}}_i(u,\varepsilon) := \| r^u_i + R(\mathcal{L}_{(u,x^i)}(\varepsilon)) \|^2  - \lambda_1 ||\varepsilon||_2^2 
\end{align}
which is the first-order approximation of~\eqref{eqn:robust_cost} in $\varepsilon$. 

Now, by using Lemma~\ref{lem:reach_set}, we find a closed form expression of $\varepsilon \in \cV^*_\rho \subset \cV_\rho$ that maximizes the cost functional~\eqref{eqn:robust_cost_L} over $\cV_\rho$ and state it in the following Proposition. Let $$m^u_{i}(t)= R\left( \Phi_{(u,x^i)}(1,t) \frac{\partial f(x,u)}{\partial u}\lvert_{x(t)=\varphi_t(u,x^i)} \right)$$ where $R$ is applied component-wise. Define the kernel ${K}^u_{i}(t,\tau)=m^u_{i}(t)^\top m^u_{i}(\tau)$. Note that it is uniformly bounded over $[0,1]^2$. Thus, we associate the bounded linear operator 
\begin{align}\label{eqn:operator_K}
      {\cal K}^u_i(t): {\cal V} \to {\cal V}: \varepsilon \mapsto \int_0^1 K^u_i(t,s) \varepsilon(s) ds.
\end{align}
With these definitions, we can state the following Proposition:
\begin{prop}\label{prop:optimal}
For a given $u \in \cV$ and a pair $(x^i,y^i)$, and $\lambda_1>0$ large enough, there exists a unique $\varepsilon^*_i(t) \in \mathcal{V}_\rho^* \subset \mathcal{V}_\rho$ that maximizes~\eqref{eqn:robust_cost_L}, and it is given by
    \begin{align}\label{eqn:epsilon_*}
        \varepsilon_i^*(t) =\rho\frac{ ( \mathcal{K}^u_i-  \lambda_1 I )^{-1}m^u_{i}(t)^\top  r^u_i }{||( \mathcal{K}^u_i -\lambda_1 I )^{-1}m^u_{i}(t)^\top r^u_i||_\infty}
    \end{align}
\end{prop}
See section~\ref{sec:proof} for the proof of Proposition~\ref{prop:optimal}. It is clear that the optimal disturbance given in~\eqref{eqn:epsilon_*} depends on the given control $u$ and the pair $(x^i,y^i)$. Therefore, we denote the disturbance that maximizes~\eqref{eqn:robust_cost} by $\varepsilon^*_i(u)$.

\paragraph{Minimization of Functional $\mathcal{J}(u,\mathcal{X})$}
 We first plug $\varepsilon^*(u)$ given in \eqref{eqn:epsilon_*} into $\mathcal{J}(u,\mathcal{X}$) so that we turn the given non-convex non-concave minimax problem into a non-convex minimization problem; precisely, we have the following approximation of cost functional to minimize over $u \in \cV$:
\begin{align}\label{eqn:updated_cost}
    \tilde{\mathcal{J}}( u , \mathcal{X} ) := \sum_{i=1}^{q} \left(  {\mathcal{J}}_i(u+ \varepsilon^*_i(u))  - \lambda_1 ||\varepsilon_i^*(u)||^2_2 \right) 
\end{align}
where we define the {\em per-sample cost functional} for the pair $(x^i,y^i)$ at $u$ as ${\mathcal{J}}_i(u) := \| R(\varphi(u ,x^i)) - y^i \|^2$ .

Now, we discuss the existence of a control $u$ that minimizes the cost functional $\tilde{\mathcal{J}}(u,\mathcal{X})$ in~\eqref{eqn:updated_cost}. 
We introduce the following subspace of $\cV$ for $i\in\mathcal{I}$:
\begin{equation}
    U(x^i,y^i):=\{ u \in \cV | \varphi(u,x^i) \in R^{-1}(y^i) \}
\end{equation}
 If it holds that $R(\varphi(u,x^i))=y^i$ for a given pair $(x^i,y^i)$, it is said that the model~\eqref{eqn:control_system} memorized the pair $(x^i,y^i)$ at the control $u$. We have the following lemma to connect the space of control functions and the optimization problem. 
\begin{lemma}
If $u \in \bigcap_{i=1}^q U(x^i, y^i)$, then $\tilde{\mathcal{J}}(u, \mathcal{X}) = 0$.
\end{lemma}
\begin{proof}
For $u \in \bigcap_{i=1}^q U(x^i, y^i)$, we have $R(\varphi(u, x^i)) = y^i$ for all $i \in \mathcal{I}$, which gives $\varepsilon^*_i(u) = 0$ by Proposition~\ref{prop:optimal}. Hence, $\tilde{\mathcal{J}}_i^{\text{Robust}}(u, 0) = 0$ and $\tilde{\mathcal{J}}(u, \mathcal{X}) = \sum_{i=1}^q \tilde{\mathcal{J}}_i^{\text{Robust}}(u, 0)$.
\end{proof}
This result shows that, under the given conditions, the optimization problem reduces to finding a control $u$ that exactly matches the given pairs via end-point mapping.

The sufficient condition for the existence of $u \in \bigcap_{i=1}^q U(x^i, y^i)$ has been studied as an application of the Chow-Rashevsky theorem~\cite{brockett2014early} in~\cite[Lemma 1]{bayram2024control}. 
In contrast, here we propose an iterative algorithm that relies on the differentiable properties of the set $\bigcap_{i=1}^q U(x^i, y^i)$. These differentiability properties were analyzed in our earlier work~\cite{bayram2025geometric} under the following assumptions:

{
\textbf{A.1:} The set of control vector fields of the $q$-folded system of the model~\eqref{eqn:control_system} is bracket-generating in $(\mathbb{R}^{\bar{n} })^{(q)}$~\cite[Lemma 1]{bayram2024control}. 

\textbf{A.2:}  The system~\eqref{eqn:control_system} has the Linearized Controllability Property at all $x^i \in \mathcal{X}$~\cite[Definition 2]{bayram2025geometric}. 

\textbf{A.3:} The sets $U(x^i,y^i)$ and $U(x^j,y^j)$ intersect transversally for any pair $(x^j,y^j),(x^i,y^i) \in (\mathcal{X},\mathcal{Y})$.
\begin{thm}
\cite[Thm 1]{bayram2025geometric}\label{thm:control_w_smp} 
Suppose that assumptions (A.1), (A.2) and (A.3) holds and let $(\mathcal{X},\mathcal{Y})$ be a paired set of cardinality $q$. Then, the space of controls $\cap_{i=1}^j U(x^i,y^i)$ is a Banach submanifold of $\cV$ of finite-codimension for $1\leq j \leq q$.
\end{thm}
This implies that the set of controls that satisfy the memorization property for $(\mathcal{X},\mathcal{Y})$ possesses a smooth geometry of a Banach submanifold. Now, we can introduce our algorithm.

\paragraph{Iterative Training of Robust nODE}

Here, we introduce our iterative learning algorithm. We denote the value of the control function (model parameters) at a given k$th$ iteration by $u^k$.

\noindent \textbf{Outer Loop-Expanding Ensemble:} We begin the algorithm with $j=0$. We progressively add the next data points $(x^{j+1}, y^{j+1})$ into the ensemble, expanding $\mathcal{X}^j$ to $\mathcal{X}^{j+1}$ in the outer loop once the inner loop achieves convergence, (i.e., $ \mathcal{J}_j(u^k+\varepsilon^*_j(u^k))$ is zero). From assumption (A.1) (i.e. LCP at $x^i$), we know that $m^u_{i}(t)$ in~\eqref{eqn:epsilon_*} is nonzero for $t\in[0,1]$. It implies that $\varepsilon_i^*(u)$ is zero only if the model~\eqref{eqn:control_system} has memorized the pair $(x^i,y^i)$ at the control $u$. Therefore, at every iteration of the outer-loop (expansion of sub-ensamble), it holds that 
$ R(\varphi( u^{k} , x^i)) = y^i   \mbox{ for all } i \leq j$.

\noindent \textbf{Inner Loop-Fixing End-Point Mapping:} Consider the new pair $(x^{j+1},y^{j+1})$. The goal of inner loop is  to minimize the per-sample cost functional for the new pair $(x^{j+1},y^{j+1})$ at $u^k + \varepsilon^*_{j+1}(u^k)$ by updating from \(u^k\) to $u^{k+1} (:= u^k + \delta u^k)$. To be more precise, the update $\delta u^k$ is chosen to ensure that 
\begin{align}\label{eqn:second_cond}
    \mathcal{J}_{j+1}(u^{k+1} +  \varepsilon^*_{j+1}(u^{k+1})) \leq \mathcal{J}_{j+1}(u^{k} +  \varepsilon^*_{j+1}(u^k))
\end{align}
We solve this problem using gradient descent. To this end, we define the first-order variation of the  per-sample cost functional $\mathcal{J}_{i}(u)$ at a control $u$ in $\delta u$ as $D_{\delta u} {\mathcal{J}}_{i}(u):= {\mathcal{J}}_{i}(u+\delta u)-\mathcal{J}_i(u)$. Then, we have the following: 
\begin{equation}
    D_{\delta u}{\mathcal{J}}_i(u) := R^\top (\delta  \varphi_t(u,x^{i} ))\left( R( \varphi(u,x^{i})) - y^{i} \right)
\end{equation}
where we recall that $\delta \varphi_t(u, 
 x^i)=\varphi_t(u+ \delta u, x^i)-\varphi_t(u,x^i)$.

However, we propose an {\em iterative} algorithm (in which the training points are introduced sequentially one at a time by the outer loop). Therefore, we propose to fix the per-sample cost functional for all pairs $(x^i,y^i)\in(\mathcal{X}^j,\mathcal{Y}^j)$ (i.e. previously learned points in earlier iterations of the outer loop) at $u^k+\varepsilon^*_i(u^k)$. To be more precise,
\begin{align}\label{eqn:fix_cost}
     \mathcal{J}_{i}( u^{k+1} + \varepsilon_i^*(u^{k+1}) ) =  \mathcal{J}_{i}( u^k + \varepsilon_i^*(u^k) ), \forall i\leq j
\end{align}
From the definition of per-sample cost functional, we can rewrite~\eqref{eqn:fix_cost} as follows:
\begin{align}\label{eqn:fixed_endpoint}
\!R(\varphi(u^{k+1}\!+\!\varepsilon^*_i(u^{k+1}), x^i))\!=\!R(\varphi(u^k\!+\!\varepsilon^*_i(u^k), x^i))
\end{align}
for all $i \leq j$. Then, we propose to select $\delta u^{k}$ such that~\eqref{eqn:second_cond} and~\eqref{eqn:fixed_endpoint} hold. To characterize such $\delta u^k$, we consider the first-order variation of endpoint $R(\varphi(\cdot,x^i))$ at control $u^k+\varepsilon^*_i(u^k)$: 
\begin{multline}\label{eqn:variation_on_perturb}
\!\delta R(\varphi(u^k\!+\!\varepsilon^*_i(u^k),x^i))\!:=\!R(\varphi(u^k+\delta u^k +\varepsilon^*_i({u^k\!+\!\delta u^k}),x^i))-R(\varphi(u^k +\varepsilon^*_i(u^k),x^i))
\end{multline}
Then, we have the following lemma:
\begin{lemma}\label{lem:variation} For a given control $u$ and a pair $(x^i,y^i)$, if it holds that $R(\varphi(u,x^i))=y^i$, then we have 
$$\delta R(\varphi(u+\varepsilon^*_i(u),x^i))  = R ( \mathcal{L}_{(u,x^i)}(\delta u) )$$
where $ \mathcal{L}_{(u,x^i)}(\cdot)$ is given in~\eqref{eqn:operator_L}.
\end{lemma}
}
See Section~\ref{sec:proof} for the proof of Lemma~\ref{lem:variation}. Suppose that the model~\eqref{eqn:control_system} has already memorized all pairs $(x^i,y^i) \in (\mathcal{X}^j,\mathcal{Y}^j)$. Then, if we select $\delta u^k$ such that
$$\delta R(\varphi(u^k+\varepsilon^*_i(u^k),x^i))=0, \forall(x^i,y^i) \in (\mathcal{X}^j,\mathcal{Y}^j),$$
then the update $\delta u^k$ satisfies~\eqref{eqn:fixed_endpoint}. From Lemma~\ref{lem:variation}, we define the kernel of $R(\mathcal{L}_{(u,x^i)}(\cdot))$ to characterize the set of such updates $\delta u^k$ as follows:
$${\ker}(u,x^i):=\operatorname{span}\{\delta u\in \cV \mid  R(\mathcal{L}_{(u,x^i)}(\delta u)) = 0 \}$$ 
In words, ${\ker}(u,x^i)$ is the set of functions $\delta u^k$ that makes $\delta R(\varphi(u^k+\varepsilon^*_i(u^k),x^i))$ zero. By construction, $\mathcal{L}_{(u,x^i)}(\cdot)$ is an operator from infinite dimension space $\cV$ to $\mathbb{R}^{\bar{n}}$ and the readout map $R$ has full row rank Jacobian. Therefore, we have an infinite-dimensional kernel. We define the intersection of ${\ker}(u,x^i)$ for all $i \leq j$ as follows:
\begin{align*}
   {\ker}(u,\mathcal{X}^j):=\operatorname{span}\{ \delta u \in \cV \mid \delta u \in\bigcap_{ x^i\in \mathcal{X}^j}{\ker}{(u,x^i)} \}  
\end{align*}

We define the projection of $D_{\delta u} \mathcal{J}_{j+1}(u+\varepsilon_i^*(u))$ on a given subspace of functions ${\ker}(u,\mathcal{X}^j)$, denoted by $\operatorname{proj}_{{\ker}(u,\mathcal{X}^j)}D_{\delta u} \mathcal{J}_{j+1}(u+\varepsilon_i^*(u))$, as the solution of the following optimization problem:
\begin{equation}
   {\arg\min}_{d(t) \in {\ker}(u,\mathcal{X}^j)} \int_0^1  | d(\tau) -  D_{\delta u} \mathcal{J}_{j+1}(u+\varepsilon_i^*(u)) |^2 d\tau 
\end{equation}

Now, we propose the selection of $\delta u^k$ as the projection of $D_{\delta u}\mathcal{J}_{j+1}(u^k+\varepsilon_i^*(u^k))$ onto $ker(u^k,\mathcal{X}^j)$ so that the update $\delta u^k$ satisfies both~\eqref{eqn:fixed_endpoint} and~\eqref{eqn:second_cond}; to be more precise,
\begin{align}\label{eqn:proj}
    \delta u^k := \operatorname{proj}_{{\ker}(u^k,\mathcal{X}^j)}D_{\delta u} \mathcal{J}_{j+1}(u^k+\varepsilon_i^*(u^k))
\end{align}

From Theorem~\ref{thm:control_w_smp}, under the given conditions (A.1),(A.2),(A.3), we have that $\cap_{i=1}^j U(x^i,y^i)$ is a finite-codimension Banach submanifold of $\cV$ for finite $j$, that is, the set ${{\ker}(u,\mathcal{X}^j)}$ is closed and continuous. Then, the projection $\operatorname{proj}_{{\ker}(u,\mathcal{X}^j)}D_{\delta u} \mathcal{J}_{j+1}(u)$ is well defined. This guarantees that the inner loop can always find an update $\delta u^k$ that satisfies~\eqref{eqn:proj}. This iterative update guarantees that the per-sample cost for all previously seen data points (i.e. $i \leq j$) remains fixed at their value for $u^k + \varepsilon^*_{i}(u^k)$ while per-sample cost for the new point is gradually reduced.

\paragraph{Numerical Method for Iterative Training of Robust nODE} 
\begin{algorithm}
\small
\caption{Iterative Training of Robust nODE}\label{alg:kernel_gradient}
\begin{algorithmic}[1]
\STATE \textbf{Initialize} {$u \gets u^0,$} $L_j \gets 0  \mbox{ for } j=1,\cdots,q$ 
\STATE  $L \gets [L_1 ; L_2; \cdots; L_q ]$ 
\FOR{$j=1$ to $q$}  
            \REPEAT            
            \STATE $L_\ell \gets R\big(\mathcal{L}_{(u, x^\ell)}\big)$, $\ell = 1, \dots, j$ \textit{(see Algorithm~1 in~\cite{bayram2024control})}

            \STATE $\mbox{Update the corresponding block } L_\ell \mbox{ in } L,\forall \ell \in \mathcal{I}^{j-1}$ 

           \STATE $\gamma \gets ( L_j^\top L_j - {\lambda}_1 I)^{-1} D_{\delta u} \mathcal{J}_{j}(u)$ 
            \STATE $\varepsilon_j(u) \gets \rho \gamma /||\gamma||_\infty$
            \STATE $u \gets u - \alpha^k \operatorname{proj}_{\mathcal{N}(L)}  D_{\delta u}\mathcal{J}_{j}(u+\varepsilon_j(u)) $ 
            \UNTIL{\textbf{convergence}}
\ENDFOR
\RETURN $L,u$
\end{algorithmic}
\end{algorithm}

Now, we provide details for our numerical method and show how to implement~\eqref{eqn:proj} iteratively. We let $u^k  \in R^{pN}$ be the discretization and vectorization of a time-varying control function. Let \( L_i \in \mathbb{R}^{n_o \times pN} \) be the output of Algorithm~1 in~\cite{bayram2024control}, which computes a numerical approximation of \( R\big( \mathcal{L}_{(u,x^i)}(\cdot) \big) \) for a given control \( u^k \) and initial point \( x^i \).

\textbf{Inner Maximization:} The matrix \( L_i \) is equal to the discretization of $m^u_{i}(t) = R( \Phi_{(u^k,x^i)}(1,t)\frac{\partial f(x,u)}{\partial u}$) over time $t \in [0,1]$. Then, the kernel \( \mathcal{K}^u_i(s,t) \) can be computed as \( L_i^\top L_i  \in \mathbb{R}^{pN \times pN}\) by discretizing the integral in~\eqref{eqn:operator_K}. In Lines 7 and 8 of Algorithm~\ref{alg:kernel_gradient}, we compute the optimal disturbance $\varepsilon^*_{i}(u^k)$ for the pair $(x^i,y^i)$ and we scale it to fit $\mathcal{V}^*_\rho$.

We note that $\varepsilon^*_i(u^k)$ is non-zero unless $R(\varphi(u^k,x^i))=y^i$. We perturb the control (parameters) in the direction of the worst-case adversarial perturbation within a small neighborhood $\cV_\rho$. As a result, the gradient of the per-sample cost functional is computed at $u^k+\varepsilon^*_i(u^k)$ in Line 9 of Algorithm~\ref{alg:kernel_gradient} while the forward equation (prediction) of the model is evaluated at $u^k$. This mismatch between  where the cost functional and the gradient are evaluated is a key ingredient for the search of a flat minima. 

\textbf{Kernel Projection:} We have $R(\mathcal{L}_{(u^k,x^i)}(\delta u^k)) \approx L_i \delta u^k$, which implies that the right kernel of the matrix \( L_i \) is the numerical approximation of \( \operatorname{\ker}(u^k, x^i) \).  Then, we column-wise concatenate $L_\ell$ for $\ell \leq j$ (previously seen points) and place the concatenated matrix into the first block of $L\in\mathbb{R}^{n_oq \times pN}$ and fill the remaining entries with zeros. One can easily see that the right kernel of the matrix $L$, denoted by $\mathcal{N}(L)$, is a numerical approximation of ${\ker}{(u^k,\mathcal{X}^j)}$ for a given $j$. Then, we apply the kernel projection given in~\eqref{eqn:proj} in Line 9 of Algorithm~\ref{alg:kernel_gradient} where $\alpha^k$ is the learning rate at the iteration $k$. Once the inner loop (see Line 5-9) converges, we iterate to the next points by Line 3 (i.e. outer loop iteration).

\section{Proof of Main Results}\label{sec:proof}

\begin{proof}[Proof of Lemma~\ref{lem:reach_set}] Recall~\eqref{eqn:operator_L}, consider the operator $\mathcal{L}_{(u,x^i)}(v)$ mapping a control $v \in \cV$ to the state at time $1$ for LTV system in~\eqref{eqn:defn_ltv}. Since $\Phi_{(u,x^i)}(1,\tau) $ and $ \frac{\partial f(x,u)}{\partial u}$ are uniformly bounded, there exist constants $A_M$ and $B_M$ such that $
\|\Phi_{(u,x^i)}(1,\tau)\|_\infty \leq A_M$ and $\| \frac{\partial f(x,u)}{\partial u}\|_\infty \leq B_M$ for almost every $\tau \in [0,1]$. Then,
\begin{align*}
\|\mathcal{L}_{(u,x^i)}(v)\|_\infty &\leq A_M B_M \int_0^1 ||v||_\infty d\tau.
\end{align*}

This shows that the linear operator ${\mathcal{L}_{(u,x^i)}}(v)$ is a bounded operator from a Banach space $\cV$ to $\mathbb{R}^n$. Next, observe that the image $\mathcal{L}_{(u,x^i)}(\mathcal{V}_\rho)$ is a subset of $\mathbb{R}^n$, and hence its dimension, say $k$, satisfies $k \leq n$, that is, $\mathcal{L}_{(u,x^i)}(\cdot)$ has finite rank $k\leq n$. Therefore, there exists a finite number of linearly independent unit norm functions $\{v_1, v_2, \ldots, v_k\} \subset \cV$ such that for any $v \in \mathcal{V}_\rho$, we have $\mathcal{L}_{(u,x^i)}(v) = \sum_{\ell=1}^k \alpha_\ell \mathcal{L}_{(u,x^i)}(v_\ell)$. Define the finite-dimensional subspace
\begin{align}\label{eqn:defn_v_rho}
    \mathcal{V}^*_\rho := \{ v(t) = \sum_{\ell=1}^k \alpha_i v_\ell(t) : a_i \in \mathbb{R} , ||\alpha||_\infty \leq \rho  \} 
\end{align}
By construction, $\mathcal{V}^*_\rho \subset \mathcal{V}_\rho$, and the operator $\mathcal{L}_{(u,x^i)}(\cdot)$ restricted to $\mathcal{V}^*_\rho$ has the same image as $\mathcal{L}_{(u,x^i)}(\cdot)$ on $\mathcal{V}_\rho$, i.e., $
\mathcal{L}_{(u,x^i)}({\mathcal{V}^*_\rho}) = \mathcal{L}_{(u,x^i)}({\mathcal{V}_\rho})$. The bound on the coefficients in~\eqref{eqn:defn_v_rho}  implies that the set of all possible coefficient vectors is bounded in $\mathbb{R}^k$. Moreover, the mapping that sends a coefficient vector $(\alpha_1, \dots, \alpha_k)$ to the function $v(t)$ is linear and continuous. Therefore, the image of a bounded (and closed) set under this continuous mapping is itself closed. Thus, $\mathcal{V}^*_\rho$ is both bounded and closed in a finite-dimensional space, and by the Heine–Borel theorem, it is compact. 
\end{proof}

\begin{proof}[Proof of Proposition~\ref{prop:optimal}] 
For a fixed given $u$ and initial point $x^i$, we denote $R(\mathcal{L}_{(u,x^i)})(\cdot))$ by $\bar{\mathcal{L}}(\cdot)$. Recall~\eqref{eqn:robust_cost_L} as:
\begin{align}\label{eqn:robust_cost_L_prop}
    \tilde{\mathcal{J}}^{\mbox{Robust}}_i(u,\varepsilon) = \| r^u_i + \bar{\mathcal{L}}(\varepsilon) \|^2  - \lambda_1 ||\varepsilon||_2^2 
\end{align}
where $r^u_i = R(\varphi(u,x^i)) - y^i$. From Lemma~\ref{lem:reach_set}, for every $\varepsilon \in \mathcal{V}_\rho$, there exists a corresponding $\varepsilon^* \in \mathcal{V}^*_\rho$ satisfying $    {\mathcal{L}}_{(u,x^i)}(\varepsilon^*) = {\mathcal{L}}_{(u,x^i)}(\varepsilon)$. For any $\varepsilon \in \mathcal{V}_\rho$, decompose it as $\varepsilon = \varepsilon^* + a$, with $\varepsilon^* \in \mathcal{V}^*_\rho$ and $a \in \ker\big({\mathcal{L}}_{(u,x^i)}\big)$ so that $\mathcal{L}_{(u,x^i)}(a)=0$. Since our robust cost functional involves an $L^2$ norm in the regularization term, we now consider the $L^2$ inner product. We take $L^2$-projection of $\varepsilon$ onto the subspace spanned by $\{v_1, \dots, v_k\}$ (see~\eqref{eqn:defn_v_rho}). Then, by the projection theorem in Hilbert spaces (here, using the $L^2$ inner product), we have $    \|\varepsilon\|_2^2 = \|\varepsilon^* + a\|_2^2 = \|\varepsilon^*\|_2^2 + \|a\|_2^2$.
Then, we plug $\varepsilon=\varepsilon^*+a$ into \eqref{eqn:robust_cost_L}.
From Lemma~\ref{lem:reach_set}, it holds that $\bar{\mathcal{L}}(a)=0$. Then, we have: 
\begin{align}
\tilde{\mathcal{J}}^{\mbox{Robust}}_i(u,\varepsilon) = \|r^u_i + \bar{\mathcal{L}}(\varepsilon^*)\|^2 - \lambda_1 \Big(\|\varepsilon^*\|_2^2 + \|a\|_2^2\Big). 
\end{align}
Since \(\|a\|_2^2 \ge 0\), we have $\tilde{\mathcal{J}}^{\mbox{Robust}}_i(u,\varepsilon) \leq \tilde{\mathcal{J}}^{\mbox{Robust}}_i(u,\varepsilon^*).$
Thus, for any $\varepsilon \in \mathcal{V}_\rho$, there exists an $\varepsilon^* \in \mathcal{V}^*_\rho$ with the same linear map output and with a smaller (or equal) regularization penalty. Since $\mathcal{V}^*_\rho \subset \mathcal{V}_\rho$, the reverse inequality holds trivially. 
Hence, we obtain 
$$
\max_{\varepsilon \in \mathcal{V}^*_\rho} \tilde{\mathcal{J}}^{\mbox{Robust}}_i(u,\varepsilon) = \max_{\varepsilon \in \mathcal{V}_\rho} \tilde{\mathcal{J}}^{\mbox{Robust}}_i(u,\varepsilon).
$$
From which, we have a quadratic cost to maximize on a compact set $\mathcal{V}_{\rho}^*$. We denote $\tilde{\mathcal{J}}^{\mbox{Robust}}_i(u,\varepsilon)$ in~\eqref{eqn:updated_cost} for a fixed $u$ by $\tilde{\mathcal{J}}^u_i(\varepsilon)$. Now, we compute the the first variation of $\tilde{\mathcal{J}}^u_i$ at $\varepsilon$ as follows:
\begin{align}\label{eqn:first_var_robust}
   \delta \tilde{\mathcal{J}}^u_i\lvert_{\varepsilon}(\eta) = \lim_{\alpha \to 0} \frac{\tilde{\mathcal{J}}^u_i(\varepsilon+\alpha\eta) - \tilde{\mathcal{J}}^u_i(\varepsilon)}{\alpha} 
\end{align}
where $\eta \in \cV_\rho^*$. From linearity of $\bar{\mathcal{L}}(\cdot)$, we have 
 $$\tilde{\mathcal{J}}^u_i(\varepsilon+\alpha\eta) = \|r^u_i+ \bar{\mathcal{L}}(\varepsilon)+\alpha\bar{\mathcal{L}}(\eta)\|^2  -  \|\varepsilon + \alpha  \eta \|^2.$$ 
 For convenience, we denote $R(\Phi_{(u,x^i)}(1,t)\frac{\partial f(x,u)}{\partial u})$ by $m(t)$. Then, we plug $R(\mathcal{L}_{(u,x^i)}(\cdot))$ (see~\eqref{eqn:operator_L}) into~\eqref{eqn:first_var_robust} to arrive at: 
\begin{align*}
  \delta  \tilde{\mathcal{J}}^u_i\lvert_{\varepsilon}(\eta)  = &   (\int_0^1 m(t) \eta(t) dt )^\top r^u_i - \lambda_1 \int_0^1 \eta(t)^\top \varepsilon(t) dt \\
    &  + (\int_0^1 \eta(t) m(t) dt )^\top (\int_0^1 m(t)\varepsilon(t) dt )  
\end{align*}
Note that $R$ is linear and it is applied component-wise. For all admissible perturbations $\eta \in \cV^*_{\rho}$, the first-order necessary condition for optimality~\cite{liberzon2011calculus} requires $\delta  \tilde{\mathcal{J}}^u_i\lvert_{\varepsilon}(\eta) = 0$. We know that $m(t),\eta(t)$ and $\varepsilon(t)$ are bounded functions on $[0,1]$. Then, from Fubini–Tonelli 
Theorem~\cite{rudin1964principles}, we have that:
\begin{align}\label{eqn:fubini}
    \int_0^1\!\eta(t)^\top\!\left(\int_0^1\!m(t)^\top\!m(s)\varepsilon(s)ds-\!\lambda_1\varepsilon(t)\!-\!m(t)^\top r^u_i\right)\!dt\!=\!0 
\end{align}
for all $\eta(t) \in \mathcal{V}_{\rho}$. From the Fundamental Lemma of the Calculus of Variations~\cite[Lemma 2.1]{liberzon2011calculus} and 
~\eqref{eqn:operator_K}, we have a Fredholm integral equation of the second kind~\cite{kress1999linear}: 
\begin{align}
      {\lambda}  \int_0^1 {K}(t,s) \varepsilon(s) ds  - {\lambda}    m(t)^\top r^u_i  =\varepsilon(t)
\end{align}
with the kernel $K(t,s)= m(t)^\top m(s)$ where ${\lambda}  =1/{\lambda_1}$. 
Since the kernel $K$ is continuous on $0\leq s \leq t \leq 1$ and on $0\leq t \leq s \leq 1$, it is weakly singular. Then, under the assumption that $\lambda_1$ is sufficiently large, we have~\eqref{eqn:epsilon_*}.

Now, we compute the second variation of $\tilde{\mathcal{J}}^u_i(\varepsilon)$ as:
\begin{align}\label{eqn:second_order}
   \delta^2 \tilde{\mathcal{J}}^u_i\lvert_{\varepsilon}(\eta) &=   \bar{\mathcal{L}}^\top(\eta)\bar{\mathcal{L}}(\eta) -\lambda_1||\eta||_2^2 = ||  \bar{\mathcal{L}}(\eta) ||^2_2 - \lambda_1||\eta||_2^2 
\end{align}
We derive a quadratic form for the second-order variation of the functional. For sufficiently large $\lambda_1 > 0$, the second variation $\delta^2 \tilde{\mathcal{J}}^u_i\lvert_{\varepsilon}(\eta)$ is uniformly negative definite, i.e., $\delta^2 \tilde{\mathcal{J}}^u_i\lvert_{\varepsilon}(\eta) < 0$ for all non-zero perturbations $\eta$. This guarantees that {\em the second-order necessary condition for optimality} is satisfied for large $\lambda_1$. Moreover, since $\varepsilon$ and $\eta$ lie in a compact set $\mathcal{V}_{\rho^*}$, the sufficient condition for optimality is also satisfied. Then, the disturbance $\varepsilon$ in~\eqref{eqn:epsilon_*} is indeed an optimal solution over the compact set $\mathcal{V}_{\rho^*}$.
\end{proof}
\begin{proof}[Proof of Lemma~\ref{lem:variation}] For the sake of simplicity, we remove the time $t$ from the notation. For a small variation of the control function $u(\tau)$, denoted by $\delta u(\tau)$, we have the following:
\begin{align}\label{eqn:new_flow_dist}
    \dot{x} + \delta \dot{x}=   f\left( x + \delta x ,  u+ \delta u + \varepsilon^*_i(u + \delta u) \right)
\end{align}
Taking the {\em first} order Taylor expansion of~\eqref{eqn:new_flow_dist} around the trajectory $x(t)=\varphi_t(u+\varepsilon_i^*(u),x^i)$ and subtracting $\varphi_t(u+\varepsilon_i^*(u),x^i)$, we get
\begin{align}
     \delta \dot x = \frac{\partial f(x,\hat u )}{\partial x} \delta x  +   \frac{\partial f(x,\hat u)}{\partial u} (   \delta \varepsilon_i^*\lvert_u (\delta u) + \delta u  )
\end{align}
where $\hat u:=u +\varepsilon_i^*(u)$ and $\delta \varepsilon_i^*\lvert_u (\delta u):=\varepsilon_i^*(u + \delta u ) -  \varepsilon_i^*(u) $, that is, $\delta \varepsilon_i^*\lvert_u (\delta u)$ is the first-order variation of $\varepsilon^*_i$ at control function $u$. 
Note that $\varepsilon_i^*$ is a function of both $x^i$ and $u$. From Proposition~\ref{prop:optimal}, we know that the disturbance $\varepsilon^*_i$ is the maximum of~\eqref{eqn:robust_cost_L} at $u$ for the point $x^i$. By the first-order necessary condition of optimality~\cite{liberzon2011calculus}, we have that $\delta \varepsilon_i^*\lvert_u (\delta u)=0$ at $u$ for the point $x^i$. 
Using the variation of constants formula~\cite{liberzon2011calculus} for $z(t)=\delta \varphi_t(u,x^i)$, we have:
\begin{align}\label{eqn:const_variation}
    z(t)\!=\!\int_{0}^{t}\Phi_{(u+\varepsilon_i^*(u),x^i)}(t,\tau)\!\frac{\partial f(x,u\!+\!\varepsilon_i^*(u))}{\partial u}(\delta u\!)d\tau
\end{align}
where $x(\tau)=\varphi_{\tau}(u+\varepsilon_i^*(u),x^i)$. Once we plug that $R(\varphi(u,x^i))=y^i$ into~\eqref{eqn:epsilon_*}, we have $\varepsilon_i^*(u)=0$. From this, we have that \eqref{eqn:const_variation} is equal to $\mathcal{L}_{(u,x^i)}(\delta u)$ in~\eqref{eqn:operator_L}. This completes the proof.\end{proof}


\begin{figure}[htbp]
    \centering
    \begin{subfigure}{0.50\linewidth}
        \centering
        \includegraphics[width=\linewidth]{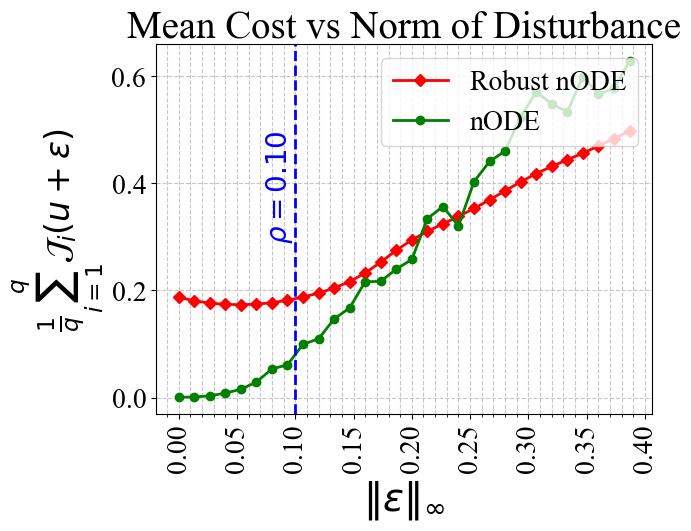}
        \caption{}
        \label{fig:cost_compare}
    \end{subfigure}
    \hfill
    \begin{subfigure}{0.44\linewidth}
        \centering
        \includegraphics[width=\linewidth]{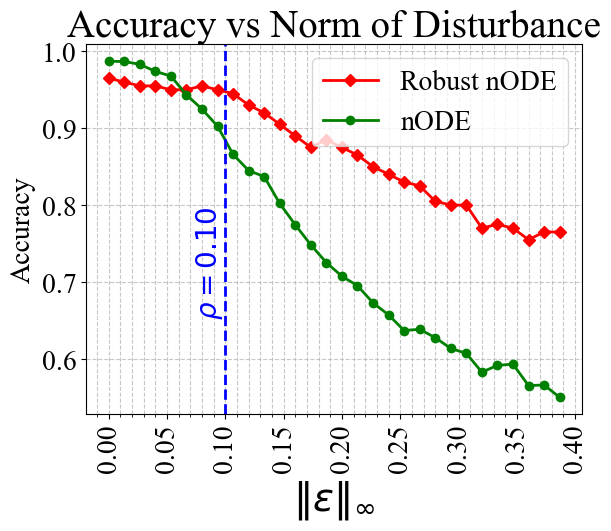}
        \caption{}
        \label{fig:acc_compare}
    \end{subfigure}

    \caption{(a) Average cost, and (b) classification accuracy as functions of the norm of the added disturbance $\|\varepsilon\|_\infty$. Red line and green line represent robust nODE and standard nODE, respectively.}
    \label{fig:plot} 
\end{figure}

\section{Numerical Results}
We design a (nonlinear) classification problem to evaluate our approach. The goal is to classify points $x_i \in \mathbb{R}^2$ as inside ($y^i=1$) or outside ($y^i=-1$) the disk of radius $r=0.5$.  
We use the uplift function $E: (x_1, x_2)^\top \mapsto (x_1, x_2, 0, 0, 0)^\top \in \mathbb{R}^5$ and define the readout map $R$ as the orthogonal projection onto the last coordinate: $R(x) = [0,0,0,0,1]\,x \in \mathbb{R}$. We consider a 100-layer residual neural network, corresponding to a discretization of $[0,1]$ with $\Delta t = 0.01$, which is an explicit Euler discretization of the following nODE:
\begin{align}\label{eqn:example_model}
    \dot{x}(t) = f(x(t), u(t)) = \tanh(W(t)x(t) + b(t)),
\end{align}  
where $\tanh$ acts component-wise, the state space is $\mathbb{R}^5$, and the control is $u(t) = (W(t), b(t)) \in L^\infty([0,1], \mathbb{R}^{5 \times 5} \times \mathbb{R}^5)$.

\textbf{Robust Neural ODE:} First, we train a model using our proposed \emph{iterative training of the Robust nODE algorithm}, outlined in Algorithm~\ref{alg:kernel_gradient}. We select \(\lambda_1 = 0.2\) and \(\rho = 0.1\). We call this model {\em Robust nODE}. 

 \textbf{Standard Neural ODE:} To provide a baseline for comparison, we train a second model without explicitly considering robustness. This model is trained to minimize the cost:  
$
    \sum_{i=1}^q \mathcal{J}_i(u) = \sum_{i=1}^q ||R(\varphi(u,E(x^i)))-y^i\|^2 
$
over $u \in \cV$. Here, the objective is to find the control \(u\) such that the predicted outputs match the targets.

We generate a test dataset \((\mathcal{X}^{\mbox{test}}, \mathcal{Y}^{\mbox{test}})\) of \(q=1000\). Robustness is evaluated via classification accuracy and average per-sample cost under disturbances \(u+\varepsilon\) with \(\|\varepsilon\|_\infty \in [0,0.4]\). Figure~\ref{fig:plot} shows that for robust nODE, the average cost stays around $0.2$ for \(\|\varepsilon\|_\infty \le 0.1\) and grows linearly beyond, while standard nODE cost rises sharply for \(\|\varepsilon\|_\infty > 0.04\).  The accuracy for robust nODE remains near $0.96$ up to \(\|\varepsilon\|_\infty = 0.1\) and decreases linearly afterwards. Standard nODE achieves higher accuracy under low disturbance (\(\|\varepsilon\|_\infty \le 0.07\)) but degrades rapidly for larger \(\varepsilon\). This demonstrates that the robust training effectively rejects disturbances below the chosen \(\rho=0.1\).

\section{Summary}
In this paper, we have introduced an iterative algorithm to solve a non-convex non-concave minimax optimization problem over an infinite-dimensional control space. Our algorithm, which is based on the two core ideas of  flat minima and kernel projection gradient descent, enables the training of robust nODEs against the control disturbance. We have evaluated our approach on a binary classification task by comparing the performances of a robust nODE and standard nODE classifiers under control disturbances.

\bibliographystyle{ieeetr}
\bibliography{learning}

\begin{thebibliography}{10}

\bibitem{dupont2019augmented}
E.~Dupont, A.~Doucet, and Y.~W. Teh, ``Augmented neural odes,'' {\em Advances in neural information processing systems}, vol.~32, 2019.

\bibitem{weinan2017proposal}
E.~Weinan, ``A proposal on machine learning via dynamical systems,'' {\em Communications in Mathematics and Statistics}, vol.~1, no.~5, 2017.

\bibitem{agrachev2020control}
A.~Agrachev and A.~Sarychev, ``Control in the spaces of ensembles of points,'' {\em SICON}, vol.~58, no.~3, pp.~1579--1596, 2020.

\bibitem{mirzadeh2020understanding}
S.~I. Mirzadeh, M.~Farajtabar, R.~Pascanu, and H.~Ghasemzadeh, ``Understanding the role of training regimes in continual learning,'' {\em NeurIPS}, vol.~33, pp.~7308--7320, 2020.

\bibitem{shi2021overcoming}
G.~Shi, J.~Chen, W.~Zhang, L.-M. Zhan, and X.-M. Wu, ``Overcoming catastrophic forgetting in incremental few-shot learning by finding flat minima,'' {\em NeurIPS}, vol.~34, pp.~6747--6761, 2021.

\bibitem{foret2020sharpness}
P.~Foret, A.~Kleiner, H.~Mobahi, and B.~Neyshabur, ``Sharpness-aware minimization for efficiently improving generalization,'' in {\em ICLR}, 2021.

\bibitem{cipriani2024minimax}
C.~Cipriani, A.~Scagliotti, and T.~W{\"o}hrer, ``A minimax optimal control approach for robust neural {ODE}s,'' in {\em 2024 ECC}, pp.~58--64, IEEE, 2024.

\bibitem{yan2019robustness}
H.~Yan, J.~Du, V.~Y. Tan, and J.~Feng, ``On robustness of neural ordinary differential equations,'' {\em arXiv preprint arXiv:1910.05513}, 2019.

\bibitem{bayram2024control}
E.~Bayram, S.~Liu, M.-A. Belabbas, and T.~Ba{\c{s}}ar, ``Control theoretic approach to fine-tuning and transfer learning,'' in {\em IFAC Symposium on Systems Theory in Data and Optimization (SysDO)}, Springer, 2025.

\bibitem{bayram2025geometric}
E.~Bayram, M.-A. Belabbas, and T.~Ba{\c{s}}ar, ``Geometric foundations of tuning without forgetting in neural {ODE}s,'' {\em arXiv preprint arXiv:2509.03474}, 2025.

\bibitem{brockett2014early}
R.~Brockett, ``The early days of geometric nonlinear control,'' {\em Automatica}, vol.~50, no.~9, pp.~2203--2224, 2014.

\bibitem{liberzon2011calculus}
D.~Liberzon, {\em Calculus of Variations and Optimal Control Theory: A Concise Introduction}.
\newblock Princeton University Press, 2011.

\bibitem{rudin1964principles}
W.~Rudin {\em et~al.}, {\em Principles of Mathematical Analysis}, vol.~3.
\newblock McGraw-Hill New York, 1964.

\bibitem{kress1999linear}
R.~Kress, {\em Linear Integral Equations}, vol.~82.
\newblock Springer, 1999.

\end{thebibliography}

\end{document}